\documentclass[review]{elsarticle}
\usepackage[dvipsnames]{xcolor}
\usepackage{epstopdf}
\usepackage{tikz}
\usepackage{xcolor}
\usetikzlibrary{patterns}
\usepackage{nomencl}
\usepackage[marginal]{footmisc}
\usepackage{pgfplots}
\usepackage{tikz}
\usetikzlibrary{plotmarks}
\usepackage{tcolorbox}
\colorlet{linecol}{black!75}

\usepackage{caption}

\usepackage{tabularx}

\usepackage{tikz}
\usetikzlibrary{calc}

\colorlet{mhpurple}{Plum!80}

\usepackage{threeparttable}

\usepackage{lineno}
\usepackage{amssymb,amsmath}
\journal{}
\usepackage{geometry}
\usepackage{graphicx}
\usepackage{booktabs}
\usepackage{color}
\usepackage{colortbl}
\usepackage{epstopdf}
\usepackage{amsmath}
\usepackage{algpseudocode}

\usepackage{amsthm}                   
\newtheoremstyle{mystyle}
{3pt}    
{3pt}    
{}       
{}       
{\bfseries} 
{.}      
{.5em}   
{}       
\theoremstyle{mystyle}             
\newtheorem{definition}{Definition}
\newtheorem{theorem}{Theorem}

\newtheorem{lemma}{Lemma}
\usepackage{multirow}
\usepackage{graphicx}
\usepackage{subfigure}
\newtheorem{property}{Property}

\usepackage{graphicx}
\usepackage{multirow}
\usepackage{tabularx}
\usepackage{cleveref}
\usepackage{titlesec}
\titlespacing*{\section}
{0pt}{0.5ex plus 1ex minus .1ex}{0.1ex plus .1ex}
\titlespacing*{\subsection}
{0pt}{0pt}{0.2pt}

\usepackage{setspace}
\usepackage[linesnumbered,ruled,vlined]{algorithm2e}

\begin{document}

\begin{frontmatter}

\title{Grey-informed neural network for time-series forecasting} 


\author[label1]{Wanli  Xie}
\author[label4]{Ruibin Zhao}
\author[label1]{Zhenguo Xu  \corref{cor1}}
\cortext[cor1]{Corresponding  author.}\ead{qfnuxzg@163.com}
\author[label1]{Tingting Liang}
\address[label1]{School  of  Communication,  Qufu  Normal  University,  Rizhao, China}
\address[label4]{School of Computer Science and Information Engineering, Chuzhou Univeristy, Chuzhou, China}


\begin{abstract}
Neural network models have shown outstanding performance and successful resolutions to complex problems in various fields. However, the majority of these models are viewed as black-box, requiring a significant amount of data for development. Consequently, in situations with limited data, constructing appropriate models becomes challenging due to the lack of transparency and scarcity of data. To tackle these challenges, this study suggests the implementation of a grey-informed neural network (GINN). The GINN ensures that the output of the neural network follows the differential equation model of the grey system, improving interpretability. Moreover, incorporating prior knowledge from grey system theory enables traditional neural networks to effectively handle small data samples. Our proposed model has been observed to uncover underlying patterns in the real world and produce reliable forecasts based on empirical data.
\end{abstract}
\end{frontmatter}

   \begin{spacing}{0.3} 
\tableofcontents
   \end{spacing}
\section{Introduction}
Due to continuous advancements, neural networks have showcased remarkable capabilities across various fields, greatly contributing to the progress of artificial intelligence technologies. Artificial neural network models have been successfully employed in time series modeling \cite{ismail2019deep}, image processing \cite{wang2020deep}, and natural language processing \cite{bai2021speaker}. In order to effectively tackle complex practical problems, numerous scholars have made fruitful improvements to artificial neural network models, resulting in a multitude of significant findings \cite{bottcher2022ai}. Among various research domains, the prediction of time utilizing artificial neural networks has garnered significant scholarly attention \cite{wang2022improved}. Present research extensively utilizes the robust modeling capabilities of neural networks to analyze real-world data and forecast trends within time series. However, mainstream neural network models require large volumes of modeling data and lack interpretability. Consequently, numerous scholars have proposed effective techniques to address these challenges. With the advancement of grey system theory \cite{wei2023nonlinear, xie2023novel}, it is anticipated that this issue will be effectively addressed. Data modeling in small sample environments can be effectively addressed through the use of grey system theory. The concept of a white system denotes a system with complete information, while a grey system contains a mix of known and unknown information, and a black system is entirely unknown. Predicting system behavior does not rely on prior data in the grey system approach.
Within grey system theory, certain real-world development patterns are attributed to grey system dynamics. By employing specific operators, the underlying laws governing system evolution can be fully elucidated. This modeling approach is currently experiencing rapid advancement, distinguishing itself from traditional statistical and fuzzy mathematical methods in its unique methodology. Figure \ref{fig:graph_one_order} shows the relationship between the three systems and the amount of data.
Chen et al. propose the Grey Neural Network (GNN), which combines grey system theory with neural network methods. The GNN model incorporates a grey layer before the neural input layer and n white layers after the neural output layer. By leveraging the strengths of both the grey model and neural network, the GNN achieves enhanced precision in forecasting. The GNN presents a promising approach for accurate predictions in various applications \cite{chen2003traffic}.
Wu et al. introduce a novel wave energy forecast model that combines an improved grey BPNN with a modified ensemble empirical mode decomposition (MEEMD)-autoregressive integrated moving average (ARIMA) approach. This integration allows the proposed model to achieve higher accuracy in predicting wave energy output \cite{wu2021combined}.
Lei et al. propose a new grey forecasting model called the neural ordinary differential grey model (NODGM), inspired by neural ordinary differential equations (NODE). Leveraging the latest techniques in NODE research, the NODGM model represents an innovative approach to grey forecasting with potential for practical application \cite{lei2021neural}.
Ma et al. adopt the concept of "Grey-box" modeling to maximize the benefits of a deterministic structure. They develop a neural grey system model that combines existing information with novel neural network techniques. The Levenberg-Marquardt algorithm is employed to facilitate model training, while Bayesian regularization is utilized to automatically adjust the regularized parameter, further enhancing the accuracy and robustness of the proposed model \cite{ma2021novel}.
Xie et al. present a fractional-order neural grey system model with a three-layer structure to maximize the advantages of each element. The input of the network is a fractional-order cumulative sequence, while the output is a predicted value. By merging these techniques with traditional grey system theory, the model offers a broader and more comprehensive understanding of the system under study \cite{xie2023fractional}. As a result of the Richards equation introduced by He et al., grey modeling has been conducted to simulate traffic parameter development trends. The method is then fused with the ability to adjust error feedback as well as complex nonlinear fitting of neural networks in order to estimate the grey model parameters and forecast traffic volatility \cite{he2023neural}. 

In order to train neural networks, a large amount of data is required, and neural networks must search the data for patterns completely. Insufficient data, however, makes it impossible to develop a suitable model. It is therefore necessary to have a large amount of data in order to train neural networks using data-driven techniques. As a result of the efforts of researchers, this problem is being addressed. To guide neural network training, some scholars have considered using physics laws as prior knowledge in order to ensure that the inputs and outputs of the model meet certain physics laws \cite{raissi2019physics}. It should also be noted that such ideas are also well suited for solving prediction problems. According to Du et al., a new underwater acoustic field prediction method based on a physics-informed neural network, called UAFP-PINN, has been developed \cite{du2023research}.  Li et al. developed a physics-informed dynamic mode decomposition method (piDMD), which incorporates the mass conservation law into a purely data-driven dynamic mode decomposition method. The spatial and temporal dynamics of solid volume fraction distribution in bubbling fluidized beds can be quickly predicted \cite{li2024physics}. Furthermore, there are also some novel methods that integrate prior knowledge into neural network models, such as embedding causal relationships into neural network models to solve thorny engineering issues \cite{iglesias2024causally}. In this paper, we consider embedding the grey system model into the neural network model from the perspective of systems theory, allowing the neural network to not only mine patterns from the data, but also ensure that the inputs and outputs of the neural network meet the patterns described by the grey system.

The research is structured into several distinct sections. The second section of the paper provides essential background information, including an in-depth analysis of the grey prediction model. Section 3 presents the GINN model. The fourth section introduces a modified version of the GINN model utilizing fractional order difference. The fifth and sixth sections demonstrate the practical application of the model in real-world scenarios and provide evidence of its effectiveness. A comprehensive summary of the entire text is presented in the seventh section as a conclusion.
\begin{table}[htbp]\scriptsize
	\centering
	\begin{tabular}{|ccl|}
		\hline
		\multicolumn{3}{|l|}{\textbf{Nomenclature}} \\
				$\mathcal{T}$ & Sample space & The set of all possible outcomes of an experiment \\ 
				$\mathcal{D}$ & Label space & The set of possible labels or output values \\ 
				$\theta$ & Parameters & The weights and biases of the neural network model \\
				$\xi$ & Weighting coefficient & A coefficient to balance the errors of the neural network and grey prediction models \\
				$\alpha$ & Fractional order & The fractional order used in the differential equations of the grey prediction model \\ 
				$\beta$ & Parameter & A positive parameter in the truncated M-fractional operators \\ 
				$a, b$ & Indices & Indices for arrays or sequences \\
				$x(t), y(t)$ & Observation sequences & The original sequence and the output or prediction sequence for time $t$ \\
				$f(t)$ & Function & A function of the values at time $t$ \\
				$D$ & Data distribution & The distribution of training data \\
				$L^{ALL}$ & Total error & The sum of the error terms from the neural network model and the grey system model \\
				$L^{NN}$ & Neural network error & The error term of the neural network model \\
				$L^{GM}$ & Grey system error & The error term of the grey system model \\
				\hline
			\end{tabular}%
	\end{table}
	\printnomenclature
\begin{figure}
	\centering
	\includegraphics[scale=0.4]{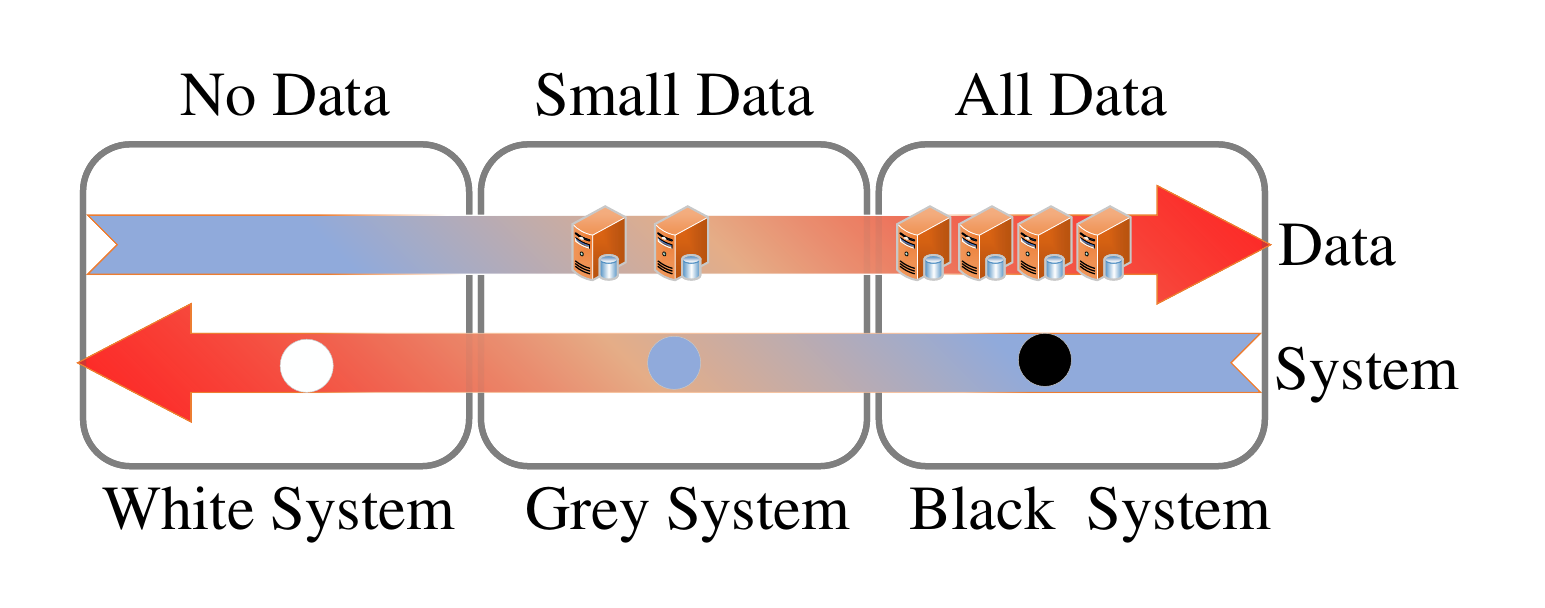}
	\caption{The relationship between the white system, grey system, and black system is intrinsically linked to the quantity of available data. When full knowledge of a system is attained, data collection becomes unnecessary to understand its operational principles. However, when only partial information is available, a grey system is required to detail the operational rules of the system.}	
	\label{fig:graph_one_order}
\end{figure}
\label{sec:introduction}
\section{Preliminaries}
Throughout this paper, ${\mathbb{N}_a} = \{ a,a + 1,a + 2, \ldots \} ,\mathbb{N}_a^d = \{a,a + 1,a + 2, \ldots ,d\}$, where $a,d \in \mathbb{R}, d - a \in {\mathbb{N}_1}$. $C ^m([0,+\infty), \mathbb{R} )$ represents a set of continuous m-order differential function from  $C ^m([0,+\infty)$ into $\mathbb{R}$.
\subsection{Truncated M-fractional derivative}
\begin{definition}[\cite{vanterler2018new}]
	The one-parameter truncated Mittag-Leffler function is defined as
	\begin{equation}
		{ }_i \mathbb{E}_\beta(z)=\sum_{k=0}^i \frac{z^k}{\Gamma(\beta k+1)},
	\end{equation}
	where $\beta>0$ and $z \in \mathbb{C}$.
\end{definition}
\begin{definition}[\cite{vanterler2018new}]
		The truncated M-fractional derivative of the function $f:[0,\infty) \to \mathbb{R}$ of order $\alpha$ is denoted as follows:
	\begin{equation}
	D_\beta ^\alpha (x):=\lim _{\varepsilon \rightarrow 0} \frac{f\left(x_i \mathbb{E}_\beta\left(\varepsilon x^{-\alpha}\right)\right)-f(x)}{\varepsilon}
	\end{equation}
	for all  $x>0$,  ${\rm{0 < }}\alpha {\rm{ \leq 1}}$ and $\beta>0$.
\end{definition}
\begin{theorem}[\cite{vanterler2018new}]If f is differentiable and  $x>0$, then
	\begin{equation}
	D_\beta ^\alpha (x)=\frac{x^{1-\alpha}}{\Gamma(\beta+1)} \frac{d f(x)}{d x},
	\end{equation}
	where $\Gamma(\cdot)$ is the Gamma function.
\end{theorem}
\begin{definition}[\cite{vanterler2018new}]
 Let $a \geq 0$ and $t \geq a$, the truncated $M$-fractional integral of order $\alpha$ for the function $f$ is formally defined by
	\begin{equation}
		\left(I_a^{\alpha, \beta} f\right)(t)=\Gamma(\beta+1) \int_a^t \frac{f(x)}{x^{1-\alpha}} d x
	\end{equation}
	for  $\beta>0$ and $0<\alpha<1$.
\end{definition}
\subsection{A brief introduction to the grey prediction model}
\label{1AGO_network}
The purpose of this subsection is to introduce the basic knowledge of grey prediction models in a brief manner.
\begin{definition}
	Let $f:{\mathbb{N}_a} \to \mathbb{R}$, the first-order difference is defined as
	\begin{equation}
		\nabla f(k) := f(k) - f(k - 1)
	\end{equation}
for $k \in {\mathbb{N}_{a + 1}}$.
\end{definition}
\begin{definition}
Set $f:{\mathbb{N}_{a + 1}} \to \mathbb{R}$ and $b \in {\mathbb{N}_a}$, then the discrete integral of $f$ is defined as
\begin{equation}
\nabla^{-1} f(k)=	\int\limits_a^b {f(k)\nabla k: = \mathop \sum \limits_{k = a + 1}^b f(k)}
\end{equation}
for $k \in {\mathbb{N}_a}$.
Specifically, we have
\begin{equation}
\nabla^{-1} f(k)=	\int _a^af(k)\nabla k: = \mathop \sum \limits_{k = a + 1}^a f(k)= 0.
\end{equation}
\end{definition}
Using the aforementioned definition as a starting point, we will now proceed to present the concept of GM(1,1) \cite{liu2016new}.
\begin{definition}
	The GM(1,1) model in continuous form can be represented as
	\begin{equation}
		\left\{ {\begin{array}{*{20}{l}}
				\begin{array}{l}
					{y}(t) = \int_1^t {{x^{}}(\tau )d} \tau ,\\
					\frac{d}{{dt}}{y^{}}(t) + {a}{y^{}}(t) = {b},
				\end{array}\\
				{{y^{}}(1) = {x^{}}(1).}
		\end{array}} \right.
		\label{CGM11}
	\end{equation}
\end{definition}
The solution to equation (\ref{CGM11}) can be obtained by performing calculations as 
\begin{equation}
	{y{}}(t) = \left( {{x^{}}(1) - \frac{b}{a}} \right){e^{ - a(t - 1)}} + \frac{b}{a}.
\end{equation}
In order to estimate parameters of the GM(1,1) model, it is necessary to discretize it.
\begin{definition}
	Let ${{X}^{(0)}} = \left\{ {{x^{}}(1),{x^{}}(2), \cdots ,{x^{}}(n)} \right\}$  is the original time series, then the GM(1,1) model in discrete form can be defined as
	\begin{equation}
\left\{ {\begin{array}{*{20}{l}}
		{\begin{array}{*{20}{l}}
				{{\nabla ^{ - 1}}x(k) = \sum\limits_{\tau  = 1}^k {x(\tau )},}\\
				{x(k) + a{\nabla ^{ - 1}}x(k) = b,}
		\end{array}}\\
		{{\nabla ^{ - 1}}x(1)= x(1)}
\end{array}} \right.
	\end{equation}
for $k \in \mathbb{N}_1^n$.
\end{definition}
Using the least squares method, the parameters of the GM(1,1) model can be determined as follows:
\begin{equation}
	{[\hat a,\hat b]^T} = {\left( {{B^T}B} \right)^{ - 1}}{B^T}Y,
\end{equation}
where
\begin{equation}
B = \left[ {\begin{array}{*{20}{c}}
		{ - \frac{1}{2}\left( {{\nabla ^{ - 1}}x(2) + {\nabla ^{ - 1}}x(1)} \right)}&1\\
		{ - \frac{1}{2}\left( {{\nabla ^{ - 1}}x(3) + {\nabla ^{ - 1}}x(2)} \right)}&1\\
		\vdots & \vdots \\
		{ - \frac{1}{2}\left( {{\nabla ^{ - 1}}x(n) + {\nabla ^{ - 1}}x(n - 1)} \right)}&1
\end{array}} \right],
\end{equation}
\begin{equation}
Y = \left[ {\begin{array}{*{20}{c}}
		{{\nabla ^{ - 1}}(2) - {\nabla ^{ - 1}}(1)}\\
		{{\nabla ^{ - 1}}(3) - {\nabla ^{ - 1}}(2)}\\
		\vdots \\
		{{\nabla ^{ - 1}}(n) - {\nabla ^{ - 1}}(n - 1)}
\end{array}} \right].
\end{equation}
Based on the estimated parameters and the discrete response function, the predicted values of the series can be calculated as follows:
\begin{equation}
	{\nabla ^{ - 1} \hat x^{}}(k) = \left( {{x^{}}(1) - \frac{{\hat b}}{{\hat a}}} \right){e^{ - \hat a(k - 1)}} + \frac{{\hat b}}{{\hat a}},k \in \mathbb{N}_1^m.
\end{equation}
As a result, the restored values can be written as follows:
\begin{equation}
	{\nabla ^{ - 1} \hat x^{}}(k) = {\nabla ^{ - 1} \hat x^{}}(k) - {\nabla ^{ - 1} \hat x^{}}(k - 1),k \in \mathbb{N}_2^m.
\end{equation}
\begin{lemma}[\cite{Bainov1992}] Set $a, b \in C(\left[t_0, T\right], \mathbb{R} )$, and  $x \in C^1\left(\left[t_0, T\right], \mathbb{R} \right)$ satisfies
\begin{equation}
	\left\{ {\begin{array}{*{20}{l}}
			{\frac{{{\rm{d}}x(t)}}{{dt}} \le a(t)x(t) + u(t),t \in \left[ {{t_0},T} \right],}\\
			{x\left( {{t_0}} \right) \le {x_0}.}
	\end{array}} \right.
\end{equation}
Then
\begin{equation}
	x(t) \leq x_0 \exp \left(\int_{t_0}^t a(\rho) d \rho\right)+\int_{t_0}^t \exp \left(\int_\rho^t a(r) d r\right) u(\rho) d \rho, \quad t \in\left[t_0, T\right] .
\end{equation}
\end{lemma}
\section{Grey-informed neural network}
\subsection{Model expression}
Let $\mathcal{X}$ denote the sample space and $\mathcal{Y}$ the label space. The distribution of training data on $\mathcal{X} \times \mathcal{Y}{\tiny }$ is denoted as $D$, with the training dataset $S=\left\{\left(x_i, y_i\right)\right\}_{i=1}^n$ comprising $n$ independent data points sampled from $D$. The model parameters are denoted by $\theta \in \Theta \subseteq \mathbb{R} ^d$. The open ball of radius $\rho>0$ centered at $\theta$ in Euclidean space is noted as $B( \theta , \rho)$, defined as $B( \theta , \rho)=\left\{ \theta ^{\prime}:\left\| \theta - \theta ^{\prime}\right\|<\rho\right\}$. The L2 norm is represented as $\|\cdot\|$.
The loss function per data point, denoted as $\ell: \Theta \times \mathcal{X} \times \mathcal{Y} \rightarrow \mathbb{R}$, is labeled as $\ell$. The empirical loss function $\hat{L}( \theta )$, is calculated as $\sum_{i=1}^n \ell\left( \theta , x_i, y_i\right)$. The gradient matrix of the function $\hat{L}(\cdot)$ at point $\theta$ are denoted as $\nabla \hat{L}( \theta )$. Furthermore, $L^{\text {oracle }}( \theta )$ is used in this paper to represent the oracle loss function. Building upon the aforementioned definition, we introduce a novel neural network error function, which is outlined as follows:
\begin{equation}
         {L^{ALL}}(\theta ) = \hat L(\theta ) + \xi {L^{GM}}(\theta )\label{baseloss}
\end{equation}
for  $ \xi  \in \left[ {0,1} \right]$, where ${L^{ALL}}(\theta)$ denotes the total error of the neural network, while $\hat L(\theta)$ represents the network's specific error. Additionally, ${L^{GM}}(\theta)$ signifies the error of the grey system, with $\xi$ serving as the weighting coefficient. The neural network model constructed with such an error function is called the GINN model.
The error function of our proposed GINN model departs significantly from that of traditional neural networks. Our novel error function is composed of two unique components: the first component represents the classic error term of neural networks, while the second component incorporates differential equations derived from grey systems. In practical computations, we will utilize a differential form for efficient computation. This dual-component approach serves a dual purpose: firstly, it allows neural networks to adhere to a data-driven mechanism, facilitating the extraction of objective patterns from data; secondly, it ensures that neural networks conform to the dynamic laws outlined by grey systems. By integrating the dynamic laws articulated by differential equations, our model introduces new prior knowledge into neural networks, enabling them to effectively model with less data.
In the neural network framework, $\hat{L}(\theta)$ can be written as $L^{NN}(\theta)$, so equation (\ref{baseloss}) can be further rewritten as
\begin{equation}
              {L^{ALL}}(\theta ) = {L^{NN}}(\theta ) + \xi {L^{GM}}(\theta ).
	\label{alllossann}
\end{equation}
If we use the Mean Squared Error (MSE) function \cite{wang2018improved} as the standard for calculating error, equation (\ref{alllossann}) can be further modified as 
\begin{equation}
                MSE = MSE _{ NN }+ \xi MSE _{ GM },
\end{equation}
where \begin{equation}
	{{\rm{MS}}{{\rm{E}}_{{\rm{NN}}}}{\rm{ = }}\frac{1}{n}\sum\limits_{i = 1}^n {{{\left| {{{y}}_{\rm{i}}^{}{\rm{ - }}{{{f}}_{\rm{i}}}} \right|}^{\rm{2}}}} ,}
	\label{errorann}
\end{equation}
\begin{equation}
	{{\rm{MS}}{{\rm{E}}_{GM}}{\rm{ = }}\frac{1}{n}\sum\limits_{i = 1}^n {{{\left| {{{y}}_{\rm{i}}^{}{\rm{ - }}{{{g}}_{\rm{i}}}} \right|}^{\rm{2}}}} .}
	\label{errorgm}
\end{equation}
In the context of a given multivariate time series input, i.e., a sliding window $X_i=\left[ x _{i-T+1}, \ldots, x _i\right] \in \mathbb{R} ^{N \times T}$, where at time $i$, the number of variables in the sequence is $N$ and the size of the sliding window is $T$, with $x _i \in \mathbb{R} ^N$ representing the corresponding $N$-dimensional multivariate values at time $i$. Therefore, the task of multivariate time series forecasting is to predict the next $\tau$ timestamps $Y_i=\left[x_{i+1}, \ldots, x_{i+\tau}\right] \in \mathbb{R} ^{N \times \tau}$ based on the historical $T$ observations. So the training set for time series problems is represented in this paper as $S = \left\{ {\left( {{{\left[ {{x_{i - T + 1}}, \ldots ,{x_i}} \right]}},{Y_i}} \right)} \right\}_{i = 1}^n$. Specifically, when $\tau  =1$, we have $S = \left\{ {\left( {{{\left[ {{x_{i - T + 1}}, \ldots ,{x_i}} \right]}},{x_{i + 1}}} \right)} \right\}_{i = 1}^n$.

\begin{figure}
	\centering
	\includegraphics[scale=0.8]{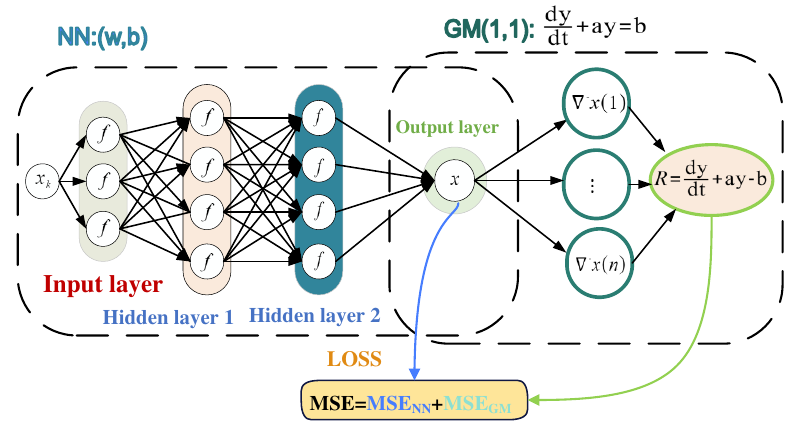}
	\caption{Schematic diagram of the network structure of GINN. The error in the proposed neural network model can be divided into two components. The first component stems from the disparity between the predicted values and the actual values within the neural network. The second component pertains to the error within the grey system model.}	
	\label{GMNN}
\end{figure}
\section{Grey-informed neural network with truncated M-fractional difference}
Based on the GINN model, we introduce a novel model known as fractional grey informed model in this section. This model is developed using a innovative fractional order accumulation operator.
\subsection{Definition of the truncated M-fractional accumulation and difference}
\begin{definition}
Considering an arbitrary function \(f: \mathbb{N}_a \rightarrow \mathbb{R}\), we define the truncated M-fractional difference (tM-D) in the following manner:
	\begin{equation}
		{\nabla ^\alpha }f(k): = \frac{{{k^{1 - \alpha }}}}{{\Gamma (\beta  + 1)}}\left( {f(k) - f(k - 1)} \right),k \in {\mathbb{N}_{a + 1}}.
	\end{equation}
\end{definition}

\begin{definition}
	Set $f: \mathbb{N}_a \rightarrow \mathbb{R}$, the truncated M-fractional accumulation (tM-A) can be defined as
	\begin{equation}
{\nabla ^{-\alpha} }f(k): = \Gamma (\beta  + 1)\int_{a + 1}^k {\frac{{f(k)}}{{{k^{1 - \alpha }}}}} \nabla t: = \Gamma (\beta  + 1)\sum\limits_{t = a}^k {\frac{{f(k)}}{{{k^{1 - \alpha }}}}} 
\label{tmaD}
	\end{equation}
for $k \in {\mathbb{N}_a}$, where $ g \leq h$ are in $\mathbb{N} _a$.
\end{definition}
\begin{property}
	Let	$F(t): = \Gamma (\beta  + 1)\int_{b+1}^t {\frac{{f(s)}}{{{x^{1 - \alpha }}}}} \nabla s$, for $ t \in \mathbb{N}_b^a$, then we have $\nabla^\alpha F(t)=f(t)$,
	where $t \in \mathbb{N}_{b+1}^a$.
	
	\begin{proof}
		Set $F(t)=\int_b^t f(s) \nabla s, \quad t \in \mathbb{N} _b^a$, then
		\begin{equation}
			\begin{aligned}
				& \nabla^\alpha F(t)=\nabla^\alpha\left(\Gamma(\beta+1) \int_{a+1}^t \frac{f(s)}{s^{1-\alpha}} \nabla s\right) \\
				& =\nabla^\alpha\left(\Gamma(\beta+1) \sum_{s=a}^t \frac{f(s)}{s^{1-\alpha}}\right) \\
				& =\frac{k^{1-\alpha}}{\Gamma(\beta+1)}\left(\Gamma(\beta+1) \sum_{s=a}^t \frac{f(s)}{s^{1-\alpha}}-\Gamma(\beta+1) \sum_{s=a}^{t-1} \frac{f(s)}{s^{1-\alpha}}\right) \\
				& =f(t)
			\end{aligned}
		\end{equation}
		for $t \in \mathbb{N} _{b+1}^a.$
	\end{proof}
\end{property}
By integrating new accumulation and difference operators, we propose a novel fractional order grey model, denoted as tM-FGM (1,1), which incorporates the fractional order integration operator. The primary objective of this study is to elucidate the fundamental architecture of the model. Parameter estimation and prediction methodologies employed in tM-FGM (1,1) are rooted in the GM (1,1) theory. Our aim is to establish a comprehensive framework for the application of this model.
\begin{definition}
		The proposed grey prediction mode with truncated M-fractional  integral in continuous form can be represented as
\begin{equation} 
\left\{ {\begin{array}{*{20}{c}}
		{y(t) = \Gamma (\beta  + 1)\int_a^t {\frac{{x(s)}}{{{s^{1 - \alpha }}}}} ds,}\\
		{y(t) + ay(t) = b,}\\
		{y(1) = x(1),}
\end{array}} \right.
\end{equation}	
where $y(t) = D_\beta ^\alpha I_a^{\alpha ,\beta }y(t)$.
\end{definition}
\begin{definition}
For $k \in \mathbb{N}_1^n$, the proposed grey prediction model in discrete form incorporating tM-A is represented as
\begin{equation}
	\left\{ {\begin{array}{*{20}{c}}
			{y(k) = \Gamma (\beta  + 1)\int_{b + 1}^{\rm{k}} {\frac{{x(\lambda)}}{{{\lambda^{1 - \alpha }}}}} \nabla \lambda,}\\
			{y(k) + ay(k) = b,}\\
			{y(1) = x(1).}
	\end{array}} \right.
\end{equation}
where ${\nabla ^\alpha }{\nabla ^{ - \alpha }}y(k)$.
\end{definition}
Our next step is to present a novel set of fractional Gronwall inequality to analyze the characteristics of solutions in grey systems according to the definition of fractional integration.
\begin{theorem}
Consider a non-negative, monotonically non-decreasing function $f(t)$ defined on the interval $[a, b]$, where $t \in [a, b]$. Let $g(t)$ be a non-negative function that satisfies
	\begin{equation}
		x(t) \le g(t) + \Gamma (\beta  + 1)\int_a^t f (\tau )x(\tau ){\tau^{\alpha  - 1}}d\tau,
	\end{equation}
	then
	\begin{equation}
\begin{aligned}
	x(t) \le g(t) + \int_a^t g (s)f(s) \exp [\Gamma (\beta  + 1)\int_s^t f (\tau ){\tau ^{\alpha  - 1}}d\tau ]{d_\alpha }s
\end{aligned}
		\label{equ:leamma1}
	\end{equation}
	for $t \in [0, + b ]$, where ${d_\alpha }{\rm{s = }}{s^{\alpha  - 1}}ds$.
	\label{newfireau}
\end{theorem}
\begin{proof}
	Set	$G(t) = g(t) + \Gamma (\beta  + 1)\int_a^t f (\tau )x(\tau ){\tau^{\alpha  - 1}}d\tau$,	then $G(a) = g(a)$ and $	x(t) \le G(t)$.
Upon calculating the fractional derivative of \( G(t) \) with respect to time on both sides, we obtain
	\begin{equation}
\begin{aligned}
	D_\beta^\alpha G(t) & =D_\beta^\alpha g(t)+D_\beta^\alpha \Gamma(\beta+1) \int_a^t f(\tau) x(\tau) \tau^{\alpha-1} d \tau \\
	& =D_\beta^\alpha g(t)+f(t) x(t).
\end{aligned}
	\end{equation}
Based on the inequality $x(t) \le G(t)$, it can be inferred that
	\begin{equation}
		D_\beta^\alpha G(t) \le D_\beta^\alpha g(t) + f(t)G(t).
		\label{daoshuxiaoyu}
	\end{equation}
By applying equation (\ref{daoshuxiaoyu}), we can multiply both sides by $\exp \left[ { - \Gamma (\beta  + 1)\int_a^t f (\tau ){\tau ^{\alpha  - 1}}d\tau } \right]$, one has
\begin{equation}
	\begin{aligned}
		& \exp \left[-\Gamma(\beta+1) \int_a^t f(\tau)(\tau-a)^{\alpha-1} d s\right] D_\beta^\alpha G(t) \\
		& \leq \exp \left[-\Gamma(\beta+1) \int_a^t f(\tau)(\tau-a)^{\alpha-1} d s\right]\left(D_\beta^\alpha g(t)+f(t) G(t)\right).
	\end{aligned}
\end{equation}
An additional level of organization can result in
\begin{equation}
	\begin{aligned}
		& \exp \left[{-\Gamma(\beta+1) \int_a^t f(\tau) \tau^{\alpha-1} d s} D_\beta^\alpha G(t)  \right]\\
		& -\exp \left[{-\Gamma(\beta+1) \int_a^t f(\tau)\tau^{\alpha-1} d s} f(t) G(t) \right] \\
			& =D_\beta^\alpha\left \{G(t) \exp \left[{-\Gamma(\beta+1) \int_a^t f(\tau) \tau^{\alpha-1} d \tau} \right ] \right\} \\
		& \leq \exp \left[ {-\Gamma(\beta+1) \int_a^t f(\tau) \tau^{\alpha-1} d \tau} D_\beta^\alpha g(t) \right].
	\end{aligned}
\label{derivate1}
\end{equation}
The properties of fractional-order integration are used to obtain
	\begin{equation}
	\begin{aligned}
		& \left.G(t) \exp \left[-\Gamma(\beta+1) \int_a^t f(\tau) \tau^{\alpha-1} d \tau\right]\right|_a ^t \\
		& =G(t) \exp \left[-\Gamma(\beta+1) \int_a^t f(\tau) \tau^{\alpha-1} d \tau\right]-G(a) \\
		& =G(t) \exp \left[-\Gamma(\beta+1) \int_a^t f(\tau) \tau^{\alpha-1} d \tau\right]-g(a)
	\end{aligned}
	\end{equation}
	Applying fractional-order integration to both sides of equation (\ref{derivate1}), we have
\begin{equation}
	\begin{aligned}
		& G(t) \exp \left[-\Gamma(\beta+1) \int_a^t f(\tau) \tau^{\alpha-1} d \tau\right] \\
		& \leq g(a)+\int_a^t \exp \left[-\Gamma(\beta+1) \int_a^s f(\tau) \tau^{\alpha-1} d \tau\right] D_\beta^\alpha g(s) d_\alpha s \\
		& =g(a)+\left.\exp \left[-\Gamma(\beta+1) \int_a^s f(\tau) \tau^{\alpha-1} d \tau\right] g(s)\right|_a ^t \\
		& -\int_a^t g(s) D_\beta^\alpha-\Gamma(\beta+1) \int_a^s f(\tau) \tau^{\alpha-1} d \tau d_\alpha s \\
		& =\exp \left[-\Gamma(\beta+1) \int_a^t f(\tau) \tau^{\alpha-1} d \tau\right] g(t) \\
		& -\int_a^t g(s) D_a^\alpha \exp \left[-\Gamma(\beta+1) \int_a^s f(\tau) \tau^{\alpha-1} d \tau\right] d_\alpha s \\
		& =\exp \left[-\Gamma(\beta+1) \int_a^t f(\tau) \tau^{\alpha-1} d \tau\right] g(t) \\
		& +\int_a^t g(s) f(s) \exp \left[-\Gamma(\beta+1) \int_a^s f(\tau) \tau^{\alpha-1} d \tau\right] d_\alpha s.
	\end{aligned}
\end{equation}
 Multiplying both sides by   $\exp [\Gamma (\beta  + 1)\int_a^t f (\tau ){\tau ^{\theta  - 1}}d\tau ]$, we have 
\begin{equation}
	\begin{aligned}
		& G(t) \leq g(t)+\exp \left[\Gamma(\beta+1) \int_a^t f(\tau) \tau^{\theta-1} d \tau\right] \\
		& \quad \quad \times \int_a^t g(s) f(s) \exp \left[-\Gamma(\beta+1) \int_a^s f(\tau) \tau^{\theta-1} d \tau\right] d_\alpha s.
	\end{aligned}
\end{equation}
Based on condition $x(t) \le G(t)$, one has
\begin{equation}
		\begin{aligned}
		&x(t) \le g(t) + \exp [\Gamma (\beta  + 1)\int_a^t f (\tau ){\tau ^{\theta  - 1}}d\tau ]\\
		& \quad \quad \times \int_a^t {g(s)f(s)\exp [ - \Gamma (\beta  + 1)\int_a^s f (\tau ){\tau ^{\theta  - 1}}d\tau ]} {d_\alpha }s.
\end{aligned}
\end{equation}
Our organization has been further enhanced by
\begin{equation}
	x(t) \le g(t) + \int_a^t {g(s)f(s)\exp \left\{ {\Gamma (\beta  + 1)\Xi } \right\}} {d_\alpha }s,
\end{equation}
where
\begin{equation}
	\Xi = \int_a^t f (\tau ){(\tau  - a)^{\theta  - 1}}d\tau  - \int_a^s f (\tau ){(\tau  - a)^{\theta  - 1}}d\tau.
\end{equation}
Based on the properties of the definite integral, the derivation of equation (\ref{equ:leamma1}) can be obtained through proper collation.
\end{proof}
\begin{theorem}
If ${y}(t)$ is viewed as the accumulated generating series for original sequence, $\hat{y}(t)$ is the fitted value of ${y}(t)$ and $f(t) = 1$. If
 \begin{equation}
 	 \hat{y}(t) - y(t) = \phi (t)>0,
 	 \label{accbudengshi} 
 \end{equation}
then
\begin{equation}
	x(t) \le \varphi (t) + \int_a^t \varphi  (s)\exp [\Gamma (\beta  + 1)\int_s^t {{\tau ^{\alpha  - 1}}d\tau } ]{d_\alpha }s.
	\label{newaccom}
\end{equation}
\label{theoremfgm}
\end{theorem}
\begin{proof}
Calculating the derivative of equation (\ref{accbudengshi}) on both sides yields the following result:
\begin{equation}
\hat{x}(t)-x(t)=\frac{t^{1-\alpha}}{\Gamma(\beta+1)} \phi^{\prime}(t).
\end{equation}
Then from (\ref{tmaD}) one has
\begin{equation}
\hat{x}(t) \le \varphi (t) + \Gamma (\beta  + 1)\int_a^t {x(\tau ){\tau ^{\alpha  - 1}}d\tau ,} 
\end{equation}
where $\frac{t^{1-\alpha}}{\Gamma(\beta+1)} \phi^{\prime}(t)=\varphi(t)$.
It follows from Theorem (\ref{newfireau}), one has
\begin{equation}
		x(t) \le \varphi (t) + \int_a^t \varphi  (s)\exp [\Gamma (\beta  + 1)\int_s^t {{\tau ^{\alpha  - 1}}d\tau } ]{d_\alpha }s.
\end{equation}
The proof of Theorem \ref{theoremfgm} is finished.
\end{proof}
This analysis indicates that the neural network model based on the tM$-$GM (1,1) model can be described as fractional-order
grey-informed neural network (FGINN), and the error function of this neural network is as follows:
\begin{equation}
	{L^{ALL}}(\theta ) = {L^{NN}}(\theta ) + \xi {L^{tM-FGM}}(\theta ).
	\label{allloss}
	\end{equation}
A schematic diagram illustrating the error fusion method between neural networks and grey system models can be found in Figure \ref{GMNN}.
The general grey information neural network modeling process is delineated in Algorithm 1. A flow chart illustrating the process from data collection to prediction can be found in Figure \ref{fig:frame1GINN}.
\begin{figure}
	\centering
	\includegraphics[width=0.7\linewidth]{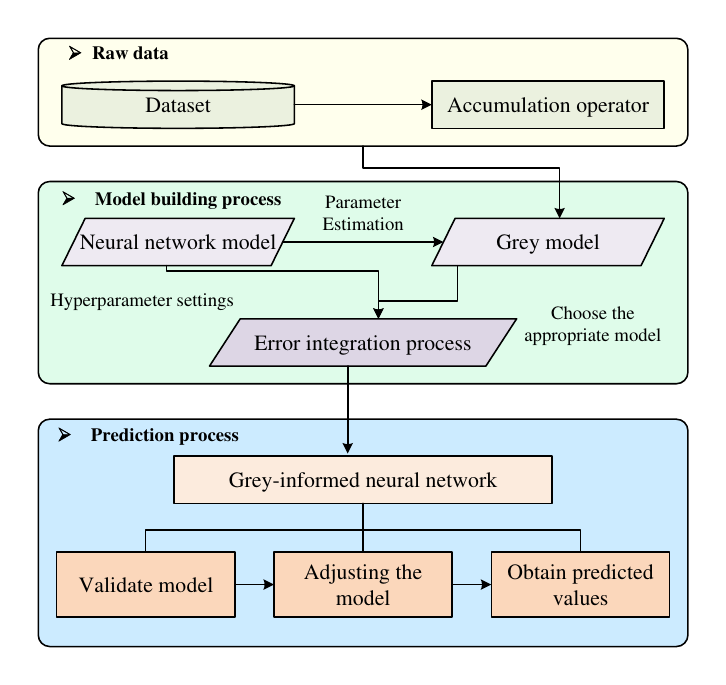}
	\caption{Framework of main research content: a comprehensive analysis.}
	\label{fig:frame1GINN}
\end{figure}

	\begin{algorithm}
	\caption{Grey-informed neural network}
	\KwIn{Training data $X$, Parameters of Grey Prediction Model $\theta$}
	\KwOut{Predicted values $\hat{Y}$}
	Initialize weights and biases in neural network\;
	Initialize learning rate $\eta$, iteration number $n_{iter}$\;
	\For{$i=1$ \KwTo $n_{iter}$}{
		Determine the prediction $\hat{Y}$  using the neural network with the current weights and biases\;
		Computer loss $L_{MSE}$ using Mean Squared Error between predicted $\hat{Y}$ and true $Y$\;
		Calculate loss $L_{Grey}$ using Grey Prediction Model with parameters $\theta$\;
		Obtain the total loss $L = L_{MSE} + L_{Grey}$\;
		Update weights and biases in neural network using gradient descent on $L$\;
	}
	\Return Predicted values $\hat{Y}$\;
				\label{algorithmGINN}
\end{algorithm}

\section{Verification of FGINN model}
In this section, we examine six practical examples in order to verify the effectiveness of the model. Table \ref{datadiscript} provides details of the relevant data. For the grey prediction model, the error term is not calculated using the square term, in order to avoid excessive error and difficulty in convergence. In this study, we employed the mean absolute error (MAE) as our error metric, which is calculated using the formula \(M{\rm{A}}{E_{GM}} = \frac{1}{n}\sum\limits_{i = 1}^{n} |{y_i} - {g_i}|\), with the weighting coefficient set to 0.1.
\begin{table}[htbp]  \scriptsize
	\caption{Information regarding the datasets obtained from the time series data library.}
	\centering
	\begin{tabular}{ll}
		\hline
		No. & Data Title \\
		\hline
		1 & Count of county hospitals (units) \\
		2 & Average number of health technical personnel per county hospital (persons) \\
		3 & Average number of beds per county maternal and child health institution (beds) \\
		4 & Average number of health technical personnel per county maternal and child health institution (persons) \\
		5 & Number of township health centers (units) \\
		6 & Proportion of rural doctors (\%) \\
		\hline
	\end{tabular}
	\caption*{\textbf{Note: }The time dimension of the data is from 2003 to 2022, with data from 2003 to 2018 used to fit the model and data from 2019 to 2022 used to test the performance of the data.}
	\label{datadiscript}
\end{table}

Utilizing a dataset obtained from the time series data library \footnote{http://stjj.guizhou.gov.cn/} spanning the years 2003 to 2022, a comparative analysis was conducted on various indices for models including FGINN, GINN, MLP \cite{zanchettin2011hybrid}, CFGM \cite{ma2020conformable}, FGM \cite{wu2013grey}, FHGM \cite{chen2020fractional}, GM \cite{mao2006application}, and DGM \cite{xie2009discrete}. The relevant experimental results are shown in Table \ref{tab:addtestpro}. The performance of the model was assessed through a comprehensive set of metrics, including Mean Absolute Percentage Error (MAPE), Mean Squared Error (MSE), Mean Absolute Error (MAE), and Root Mean Squared Error (RMSE). In most cases, GINN proved to be superior to grey prediction models or artificial neural network based on the experimental results. However, GINN's predictive performance may sometimes be lower than that of traditional models. In light of this, it is evident that GINN still has room for improvement. Based on the experimental findings, the FGINN model yielded the highest accuracy, achieving the lowest values across all evaluated metrics. This result underscores the model's precision for this particular test. This experiments further reinforced FGINN's dominance, consistently outperforming other models in terms of predictive accuracy across different datasets. A detailed analysis of the results revealed that FGINN consistently outperformed GINN across all six experiments, showcasing lower error values in MAPE, MSE, MAE, and RMSE. Notably, FGINN displayed a considerable improvement in predictive accuracy compared to GINN, as evidenced by consistently lower error values in all metrics across various datasets. For instance, in the third dataset, FGINN achieved a MAPE of 4.28894\% compared to GINN's 4.758044\%, highlighting FGINN's superior forecasting capabilities. The comprehensive comparison of error metrics underscores FGINN's robustness as a superior forecasting model compared to GINN. These findings emphasize the significance of advanced modeling techniques, like FGINN, in enhancing predictive accuracy and reliability in time series forecasting applications. Researchers and practitioners can leverage these insights to enhance forecasting methodologies and achieve more precise predictions in diverse domains.

\begin{table}[htbp] \scriptsize
	\centering
	\caption{Validation results of FGINN, GINN, GM, DGM, CFGM, FGM, FHGM and MLP with the benchmark data sets.}
	\begin{tabularx}{\textwidth}{XXXXXXXXXX}
		\toprule
		\multicolumn{1}{l}{Number} & Indices & \multicolumn{1}{l}{FGINN} & \multicolumn{1}{l}{GINN} & \multicolumn{1}{l}{MLP} & \multicolumn{1}{l}{CFGM} & \multicolumn{1}{l}{FGM} & \multicolumn{1}{l}{FHGM} & \multicolumn{1}{l}{GM } & \multicolumn{1}{l}{DGM} \\
		\midrule
		\multicolumn{1}{l}{1} & MAPE  & \textbf{0.62327} & 0.737393 & 0.786645 & 3.3247 & 3.1655 & 3.4552 & 4.8882 & 4.8484 \\
		& MSE   & 0.299435 & \textbf{0.28508} & 0.420492 & 4.6776 & 4.2505 & 5.0529 & 10.11 & 9.9481 \\
		& MAE   & \textbf{0.39554} & 0.470469 & 0.500042 & 2.1316 & 2.0294 & 2.2151 & 3.1332 & 3.1077 \\
		& RMSE & 0.547206 & \textbf{0.53393} & 0.648453 & 2.1628 & 2.0617 & 2.2479 & 3.1796 & 3.1541 \\
		&       &       &       &       &       &       &       &       &  \\
		\multicolumn{1}{l}{2} & MAPE  & \textbf{3.68056} & 3.70855 & 3.771024 & 14.177 & 17.517 & 14.177 & 14.177 & 14.563 \\
		& MSE   & \textbf{599.515} & 607.2719 & 625.7444 & 10954 & 14927 & 10954 & 10954 & 11390 \\
		& MAE   & \textbf{22.1485} & 22.31497 & 22.68692 & 86.791 & 106.57 & 86.791 & 86.791 & 89.085 \\
		& RMSE & \textbf{24.485} & 24.64289 & 25.01488 & 104.66 & 122.18 & 104.66 & 104.66 & 106.72 \\
		&       &       &       &       &       &       &       &       &  \\
		\multicolumn{1}{l}{3} & MAPE  & \textbf{4.28894} & 4.758044 & 5.53546 & 14.945 & 15.722 & 14.945 & 14.945 & 15.29 \\
		& MSE   & \textbf{14.226} & 16.36143 & 20.43448 & 147.36 & 160.61 & 147.36 & 147.36 & 153.69 \\
		& MAE   & \textbf{3.05837} & 3.406374 & 3.984379 & 11.451 & 12.025 & 11.451 & 11.451 & 11.711 \\
		& RMSE & \textbf{3.77174} & 4.044927 & 4.520451 & 12.139 & 12.673 & 12.139 & 12.139 & 12.397 \\
		&       &       &       &       &       &       &       &       &  \\
		\multicolumn{1}{l}{4} & MAPE  & \textbf{3.1733} & 3.224458 & 3.407433 & 13.998 & 5.2095 & 13.998 & 13.998 & 13.767 \\
		& MSE   & \textbf{33.7236} & 34.3879 & 37.52559 & 221.79 & 69.59 & 221.79 & 221.79 & 215.29 \\
		& MAE   & \textbf{3.65579} & 3.703505 & 3.879494 & 14.432 & 5.8065 & 14.432 & 14.432 & 14.203 \\
		& RMSE & \textbf{5.80721} & 5.864119 & 6.125814 & 14.893 & 8.3421 & 14.893 & 14.893 & 14.673 \\
		&       &       &       &       &       &       &       &       &  \\
		\multicolumn{1}{l}{5} & MAPE  & \textbf{0.7304} & 0.73226 & 0.731297 & 2.4883 & 2.4977 & 2.5206 & 3.1693 & 3.1678 \\
		& MSE   & 156.3734 & \textbf{155.9875} & 157.3502 & 1202.9 & 1210.9 & 1232.1 & 1914.1 & 1912.4 \\
		& MAE   & \textbf{10.0261} & 10.05231 & 10.03806 & 34.168 & 34.297 & 34.611 & 43.521 & 43.501 \\
		& RMSE & 12.50493 & \textbf{12.4895} & 12.54393 & 34.683 & 34.799 & 35.102 & 43.75 & 43.73 \\
		&       &       &       &       &       &       &       &       &  \\
		\multicolumn{1}{l}{6} & MAPE  & \textbf{8.27231} & 10.18479 & 9.245461 & 23.848 & 22.643 & 23.848 & 23.848 & 23.89 \\
		& MSE   & \textbf{104.757} & 135.5512 & 127.926 & 540.12 & 489.77 & 540.12 & 540.12 & 541.85 \\
		& MAE   & \textbf{7.44252} & 9.138746 & 8.332684 & 21.923 & 20.828 & 21.923 & 21.923 & 21.962 \\
		& RMSE & \textbf{10.2351} & 11.64264 & 11.31044 & 23.241 & 22.131 & 23.241 & 23.241 & 23.278 \\
		\bottomrule
	\end{tabularx}%
	    \captionsetup{labelformat=empty, singlelinecheck=off} 
	\caption*{\textbf{Note: }In the table, the configuration for GINN, FGINN, and MLP models includes 10 hidden layers each, with T = 2 and MSE as the error function. The learning rate is fixed at 0.001, and the models are trained for 2000 iterations. The orders FGM, FHGM, CFGM, etc. are determined using the PSO algorithm.}
	\captionsetup{singlelinecheck=on}
		\label{tab:addtestpro}%
\end{table}%
In summary, the recently introduced FGINN and GINN models exhibit distinct advantages in real-world modeling scenarios, boasting robust generalization capabilities and adeptly handling time series predictions within small-sample contexts. Our framework synergistically merges the nonlinear approximation prowess of neural networks with the strengths of grey prediction models, tailored for small-sample prediction tasks. Empirical comparisons revealed that FGINN outperforms GINN, suggesting that the integration of fractional calculus is efficacious and capable of a spectrum of time series prediction challenges.
\section{Applications and analysis}
This study aims to assess the forecasting capabilities of the FGINN and GINN models in predicting electricity consumption in four regions of China and to compare their performance against established machine learning models. The regions under consideration include Tianjin, Liaoning, Fujian and Gansu, with data spanning from 2000 to 2019 (measured in billions of kilowatt-hours). Table \ref{tab:tets} presents the experimental results. The dataset was divided into training and testing sets for model training and performance evaluation, respectively. Various evaluation metrics, including MAPE, MSE, MAE, and RMSE, were employed to gauge the accuracy and reliability of the models in predicting electricity consumption. The FGINN model consistently outperformed all competing models in terms of accuracy across all regions and metrics, with the lowest error values observed. This underscores the superior predictive capabilities of the FGINN model in forecasting electricity consumption patterns. Although the GINN model exhibited favorable performance in certain regions and metrics, it did not consistently match the accuracy levels achieved by the FGINN model. Specific examination of the data illustrates the supremacy of the FGINN model in delivering the lowest MAPE values across all regions, reflecting its precision in predicting percentage errors. For example, in Tianjin, the FGINN model achieved a remarkable MAPE of 3.78439\%, contrasting with the GINN model’s MAPE of 10.0289\% and those of other comparator models. Notably, the FGINN model exhibited the most favorable MSE outcomes, with minimal values recorded across all regions. In Tianjin, the FGINN model yielded an MSE of 1306.26, significantly surpassing the GINN model’s MSE of 7567.113 and the outputs of alternative models. Further validation of the FGINN model’s performance reveals its dominance in MAE values, indicating closest proximity to actual consumption data. In Tianjin, the FGINN model displayed an MAE of 31.193 in contrast with the GINN model’s MAE of 82.83217. Evaluation of RMSE values underscored the FGINN model’s consistent accuracy in predicting residuals’ standard error, particularly evidenced by the lowest values recorded across most regions. These results ascertain the FGINN model as a robust forecasting tool for electricity consumption, securing a prominent position in achieving excellent  accuracy and precision compared to competitor models. While the GINN model demonstrates promise, its potential can be further refined to approach the consistently superior performance exhibited by the FGINN model. These findings carry substantial implications for enhancing decision-making processes concerning energy utilization within the examined regions of China.

\begin{table*}[htbp] \scriptsize
	\centering
	\caption{Predicted errors of eight competing models in four regions.}
	\begin{tabularx}{\textwidth}{XXXXXXXXXX}
		\toprule
		\multicolumn{1}{l}{Region} & Indices & FGINN & GINN  & MLP   & CFGM  & FGM   & FHGM  & GM    & DGM \\
		\midrule
		\multicolumn{1}{l}{Tianjin} & MAPE  & \textbf{3.78439} & 10.0289 & 8.46356 & 22.697 & 11.641 & 23.241 & 29.55 & 29.617 \\
		& MSE   & \textbf{1306.26} & 7567.113 & 5416.489 & 38891 & 10155 & 40856 & 66663 & 66934 \\
		& MAE   & \textbf{31.193} & 82.83217 & 69.8782 & 191.17 & 97.879 & 195.78 & 249.13 & 249.69 \\
		& RMSE & \textbf{36.1422} & 86.98916 & 73.5968 & 197.21 & 100.77 & 202.13 & 258.19 & 258.72 \\
		&       &       &       &       &       &       &       &       &  \\
		\multicolumn{1}{l}{Liaoning} & MAPE  & \textbf{1.72701} & 1.767999 & 1.730181 & 14.399 & 3.4823 & 14.78 & 19.517 & 19.514 \\
		& MSE   & 2208.007 & \textbf{2111.49} & 2196.884 & 1.03E+05 & 7571.6 & 1.09E+05 & 1.93E+05 & 1.93E+05 \\
		& MAE   & \textbf{38.3682} & 39.34583 & 38.44592 & 319.46 & 74.133 & 328.12 & 435.04 & 434.94 \\
		& RMSE & 46.98943 & \textbf{45.9509} & 46.87092 & 320.54 & 87.015 & 329.42 & 439.45 & 439.3 \\
		&       &       &       &       &       &       &       &       &  \\
		\multicolumn{1}{l}{Fujian} & MAPE  & \textbf{1.60132} & 2.080047 & 1.658052 & 9.648 & 3.7342 & 11.798 & 19.008 & 19.179 \\
		& MSE   & \textbf{2537.6} & 2778.401 & 2365.615 & 48002 & 7336  & 72759 & 1.93E+05 & 1.97E+05 \\
		& MAE   & \textbf{36.3702} & 45.88663 & 37.66058 & 213.65 & 81.57 & 261.99 & 423.82 & 427.58 \\
		& RMSE & \textbf{50.3746} & 52.71054 & 48.63759 & 219.09 & 85.65 & 269.74 & 439.69 & 443.38 \\
		&       &       &       &       &       &       &       &       &  \\
		\multicolumn{1}{l}{Gansu} & MAPE  & \textbf{5.89752} & 5.903156 & 5.97287 & 19.689 & 14.038 & 20.15 & 27.423 & 27.55 \\
		& MSE   & \textbf{5277.98} & 5291.477 & 5506.729 & 57923 & 29292 & 60805 & 1.15E+05 & 1.16E+05 \\
		& MAE   & \textbf{70.711} & 70.78107 & 71.67993 & 236.5 & 167.66 & 242.19 & 331.27 & 332.79 \\
		& RMSE & \textbf{72.6497} & 72.74254 & 74.20734 & 240.67 & 171.15 & 246.59 & 338.93 & 340.42 \\
		\bottomrule
	\end{tabularx}%
	    \captionsetup{labelformat=empty, singlelinecheck=off} 
\caption*{\textbf{Note: } In the table, each of the GINN, FGINN, and MLP models has 10 hidden layers, with T = 2 and MSE as the error function. Models are trained for 2000 iterations at a learning rate of 0.001. Based on the PSO algorithm, the orders FGM, FHGM, CFGM, etc. are determined.}
\captionsetup{singlelinecheck=on}
	\label{tab:tets}%
\end{table*}%

\begin{figure}
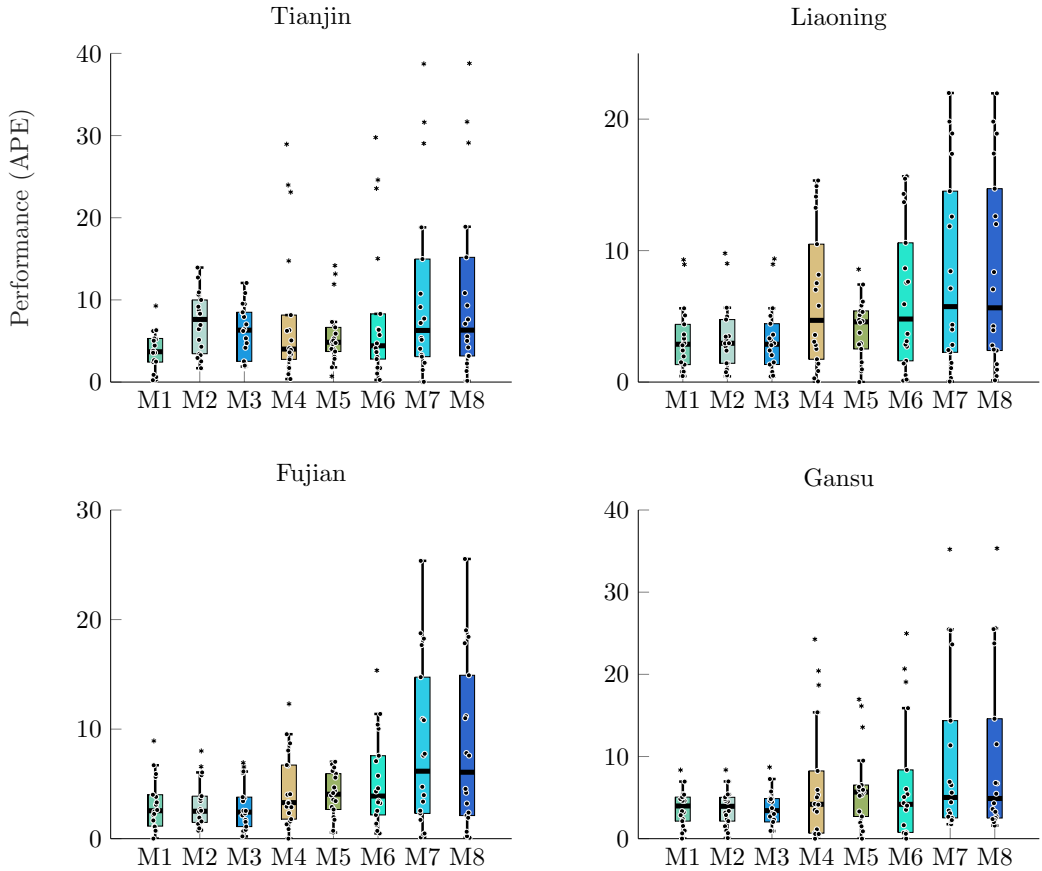

	\centering
	%
	%
	%
%
%
%
%
\definecolor{mycolor1}{rgb}{0.45000,0.80000,0.69000}%
\definecolor{mycolor2}{rgb}{0.70000,0.85000,0.82000}%
\definecolor{mycolor3}{rgb}{0.11000,0.60000,0.88000}%
\definecolor{mycolor4}{rgb}{0.85000,0.75000,0.50000}%
\definecolor{mycolor5}{rgb}{0.60000,0.70000,0.40000}%
\definecolor{mycolor6}{rgb}{0.15000,0.90000,0.80000}%
\definecolor{mycolor7}{rgb}{0.18000,0.80000,0.90000}%
\definecolor{mycolor8}{rgb}{0.18000,0.40000,0.80000}%
%
	\caption{An illustration of each point error for all models is shown in a box plot. These include "M1", "M2", "M3", "M4", "M5", "M6", "M7", and "M8" which represent FGINN, GINN, MLP, CFGM, FGM, FHGM, GM, and DGM, respectively.}
\end{figure}

\section{Conclusion}
\label{s5}
As a result, the use of grey-informed neural networks addresses the challenges associated with black-box neural network models in scenarios with limited sample sizes. GINN enhances interpretability and the ability of traditional neural networks to handle small samples by incorporating the differential equation model of grey systems. Its effectiveness in mitigating data scarcity issues in neural network modeling can be demonstrated by its use of potential underlying laws in the real world to make reasonable predictions based on actual data.
There are a number of areas that warrant further exploration in this study, even though it presents a novel approach for developing a grey neural network. Firstly, determining the optimal ratio of error terms in neural networks and grey prediction models is a promising research direction. In this paper, we are solely concerned with univariate prediction models, whereas grey models have distinct characteristics that are capable of capturing a wide range of patterns. Future research should focus on selecting an appropriate grey prediction model tailored to specific real-world problems.

\noindent\textbf{Declaration of competing interest}

The authors declare that they have no known competing financial interests or personal relationships that could have appeared to
influence the work reported in this paper.

\noindent\textbf{Data and code availability}

The code is available in \textcolor{blue}{https://github.com/Sunflowersandshells/GINN/tree/master}.

\noindent\textbf{Acknowledgments}

This work is supported by the Science and Technology Plan Project in Chuzhou  (No. 2021ZD016), the Research Projects of the Chuzhou University (No. 2022XJZD14), and the Anhui Province's University Research Projects (No. 2023AH040216).

\bibliographystyle{elsarticle-num}
\bibliography{greybib}
\end{document}